\documentclass[smallextended]{svjour3} 

\usepackage{mathptmx}
\usepackage{graphicx}
\usepackage{amsmath,amssymb}
\usepackage{subfig}
\usepackage{tikz}
\usetikzlibrary{shapes}
\usetikzlibrary{arrows}
\smartqed

\begin{document}
\title{On soft power diagrams}
\author{Steffen Borgwardt}

\institute{S. Borgwardt \at Fakult\"at f\"ur Mathematik,
Technische~Universit\"at M\"{u}nchen, 80290 M\"{u}nchen,
Germany\\\email{borgwardt@ma.tum.de}, Tel.: $+89/28916876$}

\bibstyle

\maketitle
\begin{abstract}
Many applications in data analysis begin with a set of points in a Euclidean space that is partitioned into clusters. Common tasks then are to devise a classifier deciding which of the clusters a new point is associated to, finding outliers with respect to the clusters, or identifying the type of clustering used for the partition.

One of the common kinds of clusterings are (balanced) least-squares assignments with respect to a given set of sites. For these, there is a 'separating power diagram' for which each cluster lies in its own cell.

In the present paper, we aim for efficient algorithms for outlier detection and the computation of thresholds that measure how similar a clustering is to a least-squares assignment for fixed sites. For this purpose, we devise a new model for the computation of a 'soft power diagram', which allows a soft separation of the clusters with 'point counting properties'; e.g. we are able to prescribe how many points we want to classify as outliers.

As our results hold for a more general non-convex model of free sites, we describe it and our proofs in this more general way. Its locally optimal solutions satisfy the aforementioned point counting properties. For our target applications that use fixed sites, our algorithms are efficiently solvable to global optimality by linear programming.


\end{abstract}

\section{Introduction}
The ability to extract new information from large data sets is one of the key steps in today's decision making processes. Such data sets are often represented as sets of points $X\subset \mathbb{R}^d$ in $d$-dimensional Euclidean space. In many settings, $X$ is already partitioned into clusters $C_1,\dots,C_k$. One of the frequent tasks then is the derivation of a so-called classifier, a rule that explains which of the existing clusters a new point in $\mathbb{R}^d$ should be assigned to. Further, the identification of outliers with respect to the given clusters or an identification of the clustering principles used to obtain $C_1,\dots,C_k$ are of interest.

In this paper, we are interested in data sets that are clustered as (noisy) balanced least-squares assignments with respect to a fixed set of sites $s_i\in \mathbb{R}^d$, one for each cluster. Least-squares assignments are one of the most popular clustering principles, e.g. used by the $k$-means algorithm. We will devise efficient algorithms for the above applications for this frequent special case. 

Our algorithms are based on a model that constructs a soft-separating power diagram for which each cluster lies in its own cell. The key feature of this approach is what we will call 'point counting properties' of the power diagram, which are necessary for our the applications we consider.

Our results also hold when the sites for the construction of the power diagram are not fixed beforehand. We then obtain a non-convex model, but for the sake of completeness, we present our approach to the case of fixed sites that we are interested in as a special case of this more general framework. This simpler situation arises in many real-world applications; let us motivate our contribution by such an application, and discuss in which ways we have to extend the state of the art.

\subsection{Motivation and state of the art}
Suppose there is a set of facilities $S:=\{s_1,\dots,s_k\}\subset \mathbb{R}^2$. They have to supply a set of customers $X:=\{x_1,\dots,x_n\}\subset \mathbb{R}^2$. Note that both sets are represented by their geographic locations in the Euclidean plane. 

Further, the facilities $s_i$ cannot supply more than a total of $\kappa_i^+\in \mathbb{N}$ of the customers, and may not supply less than a total of  $\kappa_i^-\in \mathbb{N}$ of them, for efficiency reasons. Finally, assume that the transport from a facility to a customer underlies a quadratic loss, i.e. the cost of supplying $x_l$ from $s_i$ can be measured by the square of the Euclidean distance $\|s_i-x_l\|^2$. A typical example for such a loss arises in energy distribution. Generally the square of the Euclidean distance is also used to treat customers far from a facility in a fair way.

A most efficient assignment $\mathcal{C}:=(C_1,\dots,C_k)$ of the customers is a {\bf balanced least-squares assignment}\footnote[1]{The terms in bold throughout this section are defined more formally in Chapters $2$ and $3$.} computed as $$\mathcal{C}:=\arg\min\limits_{(C_1,\dots,C_k)} \sum\limits_{i=1}^k\sum\limits_{x_l\in C_i} \|s_i-x_l\|^2 \text{ w.r.t } \kappa_i^- \leq |C_i| \leq \kappa_i^+ \text{ for all } i\leq k,$$
where $|C_i|$ denotes the number of points, respectively customers, in $C_i$.

Such a clustering $\mathcal{C}$ allows a special polytopal cell decomposition of the underlying space $\mathbb{R}^d$, called a {\bf power diagram} $P:=(P_1,\dots,P_k)$, such that each cluster $C_i$ lies in exactly one of the cells $P_i$ of the diagram \cite{aha-98}. Classically its cells are defined by sites $s_1,\dots,s_k\in\mathbb{R}^d$ and weights $w_1,\dots,w_k\geq 0$  as 
 $$P_i:=\{x \in \mathbb{R}^d: \|s_i-x\|^2 -w_i \leq \|s_j-x\|^2-w_j  \quad\text{ for all } j \in \{1,\dots,k\} \backslash \{i\}\}.$$
We call this a {\bf separating power diagram} for $\mathcal{C}$. It can be constructed with respect to the sites of the least-squares assignment. 
 Figure \ref{fig:PDexampleWithSites} depicts an example.  The corresponding clusterings also appear as the extreme points in the studies of special geometric bodies \cite{b-10,bg-11}.

\begin{figure}
\begin{center}
\begin{tikzpicture}[scale=0.6]
\useasboundingbox (-5,-8) rectangle (6,5);
\fill [blue] (-1.5,0.5) circle (4pt);
\fill [blue] (-1.5,-1) circle (4pt);
\fill [blue] (-3.5,0) circle (4pt);
\fill [blue] (-2.5,-1) circle (4pt);
\fill [blue] (-2,-2) circle (4pt);
\fill [blue] (-3.5,-1) circle (4pt);
\fill [blue] (-4,-1.5) circle (4pt);
\fill [blue] (-1,0) circle (4pt);
\fill [blue] (-2,-0.75) circle (4pt);

\fill [green] (4.5,-1) circle (4pt);
\fill [green] (5,1) circle (4pt);
\fill [green] (4,1) circle (4pt);
\fill [green] (3.5,1) circle (4pt);
\fill [green] (4.5,0.5) circle (4pt);
\fill [green] (4,1) circle (4pt);
\fill [green] (3,0) circle (4pt);

\fill [red] (0,-5) circle (4pt);
\fill [red] (1,-6) circle (4pt);
\fill [red] (2,-5.5) circle (4pt);
\fill [red] (1,-5) circle (4pt);
\fill [red] (-2,-5.5) circle (4pt);

\fill [yellow] (0,3) circle (4pt);
\fill [yellow] (1,4) circle (4pt);
\fill [yellow] (2,3.5) circle (4pt);
\fill [yellow] (1,3) circle (4pt);
\fill [yellow] (-1,3) circle (4pt);
\fill [yellow] (-1.5,4) circle (4pt);

\draw (-3,3) -- (0,1);
\draw (4,3) -- (0,1);
\draw (0,1) -- (1.5,-2);
\draw (1.5,-2) -- (-3.5,-4);
\draw (1.5,-2) -- (5, -5);


\fill [yellow] (1.5,3) circle (7pt);
\fill [blue] (-1.5,-1.5) circle (7pt);
\fill [green] (2.7,0.6) circle (7pt);
\fill [red] (-0.74,-3.41) circle (7pt);
\end{tikzpicture}
  \end{center}
\caption{A separating power diagram for a balanced least-squares assignment in $\mathbb{R}^2$. The large dots are the sites (or facilities) used for both the least-squares assignment and the construction of the diagram.}
\label{fig:PDexampleWithSites}
\end{figure}
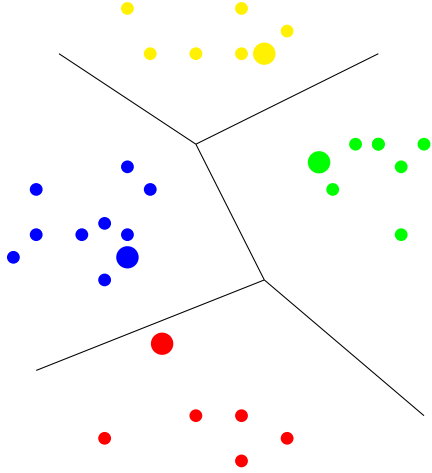

Power diagrams are a classical data structure in computational geometry, and generalize the well-known Voronoi diagrams. They arise in many applications. See \cite{a-87} for a survey. In machine learning, they are the classifiers of the so-called alltogether models for multiclass support vector machines \cite{v-98,ww-98,bb-99,cs-01}.  In the literature, these kinds of classifiers also appear as piecewise-linear separability \cite{bm-92} and full $S$-induced cell decompositions \cite{b-10}.

Alternatively, a power diagram can be defined by a special set of hyperplanes to separate the cells from each other. It is natural and common practice to use this set of hyperplanes to obtain a classifier for the different clusters: A new customer $x\in \mathbb{R}^2$ is assigned to the facility of the cell that the customer lies in.  In the context of our example, we are interested in finding a 'best' classifier for the assignment of new customers to one of the existing facilities.

The quality of such a classifier is intuitively measured by the {\bf margin}, which is the smallest Euclidean distance of a point to the boundary of its cell: The larger the margin, the better a classifier typically performs in practice.  Figure \ref{fig:PDexampleWithMargin} depicts the margin for the example in Figure \ref{fig:PDexampleWithSites}  by the width of the gray area around the hyperplanes. For the presented example, a classification task then would be to compute a separating power diagram of optimal margin for a given balanced least-squares assignment.

\begin{figure}
\begin{center}
\begin{tikzpicture}[scale=0.6]
\useasboundingbox (-5,-8) rectangle (6,5);
\draw[line width=1.33cm, color=gray] (-3,3) -- (0,1);
\draw[line width=1.33cm, color=gray]  (4,3) -- (0,1);
\draw[line width=1.33cm, color=gray]  (0,1) -- (1.5,-2);
\draw[line width=1.33cm, color=gray]  (1.5,-2) -- (-3.5,-4);
\draw[line width=1.33cm, color=gray]  (1.5,-2) -- (5, -5);
\draw (-3,3) -- (0,1);
\draw (4,3) -- (0,1);
\draw (0,1) -- (1.5,-2);
\draw (1.5,-2) -- (-3.5,-4);
\draw (1.5,-2) -- (5, -5);

\fill [blue] (-1.5,0.5) circle (4pt);
\fill [blue] (-1.5,-1) circle (4pt);
\fill [blue] (-3.5,0) circle (4pt);
\fill [blue] (-2.5,-1) circle (4pt);
\fill [blue] (-2,-2) circle (4pt);
\fill [blue] (-3.5,-1) circle (4pt);
\fill [blue] (-4,-1.5) circle (4pt);
\fill [blue] (-1,0) circle (4pt);
\fill [blue] (-2,-0.75) circle (4pt);

\fill [green] (4.5,-1) circle (4pt);
\fill [green] (5,1) circle (4pt);
\fill [green] (4,1) circle (4pt);
\fill [green] (3.5,1) circle (4pt);
\fill [green] (4.5,0.5) circle (4pt);
\fill [green] (4,1) circle (4pt);
\fill [green] (3,0) circle (4pt);

\fill [red] (0,-5) circle (4pt);
\fill [red] (1,-6) circle (4pt);
\fill [red] (2,-5.5) circle (4pt);
\fill [red] (1,-5) circle (4pt);
\fill [red] (-2,-5.5) circle (4pt);

\fill [yellow] (0,3) circle (4pt);
\fill [yellow] (1,4) circle (4pt);
\fill [yellow] (2,3.5) circle (4pt);
\fill [yellow] (1,3) circle (4pt);
\fill [yellow] (-1,3) circle (4pt);
\fill [yellow] (-1.5,4) circle (4pt);
\end{tikzpicture}
  \end{center}
\caption{A separating power diagram for four clusters in $\mathbb{R}^2$. The width of the gray area around the hyperplanes of the diagram depicts the margin.}
\label{fig:PDexampleWithMargin}
\end{figure}
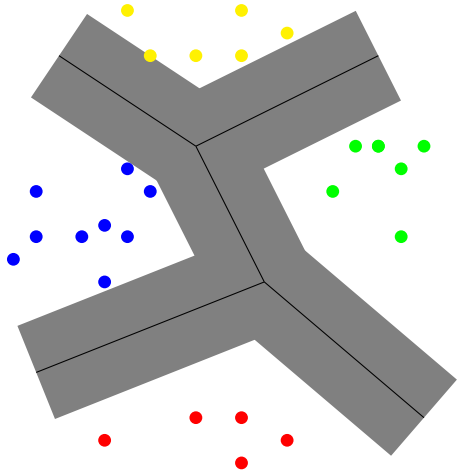

On the other hand, most real-world data sets are not 'that nice': While least-squares assignments are one of the widely-used clustering concepts,  a lot of data sets are noisy and have misclassified points. Further, there often are intentional exceptions. In our example, there may be contracts that fix that certain customers have to be supplied from a certain facility. Figure \ref{fig:ProblemSituation} depicts an example for this situation. Except for a few 'misclassified points', the clustering is identical to the least-squares assignment in Figure \ref{fig:PDexampleWithSites}. Still, a separating power diagram cannot be constructed.

Here some natural questions arise: Which customer assignments prevent the construction of a separating power diagram? Which customer assignments are 'worst' in the supply plan? These questions are intimately related to outlier detection. Answers would provide opportunites for improving the current supply plan. If there are a lot of bad assignments of customers to facilities, one could ask even more generally: How 'similar' is the supply plan actually still to a least-squares assignment? If a viable measure for this similarity (which we come up with in this paper) is low, it may be best to come up with a completely new supply plan.

Answers to these questions are at the core of the present paper. For this, one has to use to the concept of {\bf soft-margin separation}, which allows for the construction of a classifier despite having some misclassified points. Let us briefly turn to (multiclass) support vector machines which use this concept. 

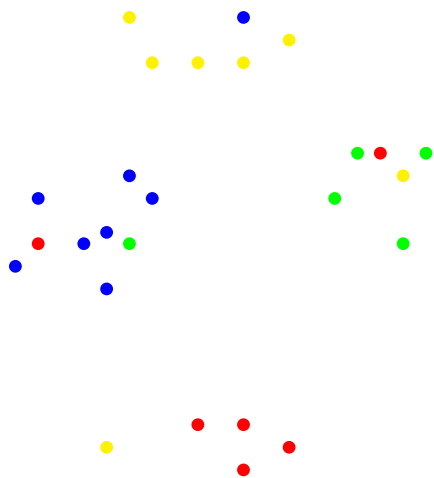
\begin{figure}
\begin{center}
\begin{tikzpicture}[scale=0.6]
\useasboundingbox (-5,-8) rectangle (6,5);
\fill [blue] (-1.5,0.5) circle (4pt);
\fill [green] (-1.5,-1) circle (4pt);
\fill [blue] (-3.5,0) circle (4pt);
\fill [blue] (-2.5,-1) circle (4pt);
\fill [blue] (-2,-2) circle (4pt);
\fill [red] (-3.5,-1) circle (4pt);
\fill [blue] (-4,-1.5) circle (4pt);
\fill [blue] (-1,0) circle (4pt);
\fill [blue] (-2,-0.75) circle (4pt);

\fill [green] (4.5,-1) circle (4pt);
\fill [green] (5,1) circle (4pt);
\fill [red] (4,1) circle (4pt);
\fill [green] (3.5,1) circle (4pt);
\fill [yellow] (4.5,0.5) circle (4pt);
\fill [green] (3,0) circle (4pt);

\fill [red] (0,-5) circle (4pt);
\fill [red] (1,-6) circle (4pt);
\fill [red] (2,-5.5) circle (4pt);
\fill [red] (1,-5) circle (4pt);
\fill [yellow] (-2,-5.5) circle (4pt);

\fill [yellow] (0,3) circle (4pt);
\fill [blue] (1,4) circle (4pt);
\fill [yellow] (2,3.5) circle (4pt);
\fill [yellow] (1,3) circle (4pt);
\fill [yellow] (-1,3) circle (4pt);
\fill [yellow] (-1.5,4) circle (4pt);


\end{tikzpicture}
  \end{center}
\caption{A clustering of four clusters in $\mathbb{R}^2$, which only differs slightly from the clusterings in Figures \ref{fig:PDexampleWithSites} and \ref{fig:PDexampleWithMargin}, but for which there is no separating power diagram.}
\label{fig:ProblemSituation}
\end{figure}

Binary classification tasks (i.e. $k=2$) have been studied well. One uses penalty terms for misclassified points to bound and control the number of {\bf support vectors} and {\bf margin errors} that are allowed in the construction of a separating hyperplane for the two clusters \cite{cv-95,sswb-00}: Informally, support vectors are the points of the data set whose removal would change the optimal separating hyperplane, margin errors are the points which lie within a distance of at most the margin of the separating hyperplane or are on the wrong side of the separating hyperplane. Note that there is a direct tradeoff in between the margin and the number of margin errors; the larger the margin, the more margin errors exist. 

For the case $k>2$ that we are interested in, soft margin separation is much more complicated. See \cite{hl-02} for a short survey. This begins with different possible definitions of what margin errors are: A first approach is to measure the misclassification of a point with respect to each separating hyperplane \cite{v-98,ww-98}. A second one is to only consider the 'worst' violation of a separating hyperplane by a point \cite{cs-01}.

Recall that alltogether models for multiclass classification construct power diagrams and allow for soft-margin separation \cite{v-98,ww-98,cs-01}. Let us explain why transferring these to our simpler setting of fixed sites does not help us: They do not use a shared margin for the cluster pairs, but instead optimize the sum of squared pairwise margin sizes. This fact means that one loses the ability to quantitatively compare the misclassification of points with each other, which is fundamental for our applications in outlier detection and evaluting the similarity of the underlying clustering to a least-squares assignment (and actually for assessing the quality of classifiers themselves, as well -- see e.g. \cite{tkht-09,tkht2-09}). Additionally, we desire further special 'point counting properties' for the power diagrams we construct, and these cannot be achieved by transferring existing approaches. 

These are the key reasons for coming up with a new model. We now exhibit our results and give a brief outline.


\subsection{Results and outline}

In the present paper, we provide a model for computing a separating power diagram that implements a shared margin among cluster pairs, for the two natural interpretations of margin errors in a multiclass scenario. We call the classifiers that we construct {\bf soft power diagrams}. Figure \ref{fig:SoftPD} depicts an example for the clustering of Figure \ref{fig:ProblemSituation}. 

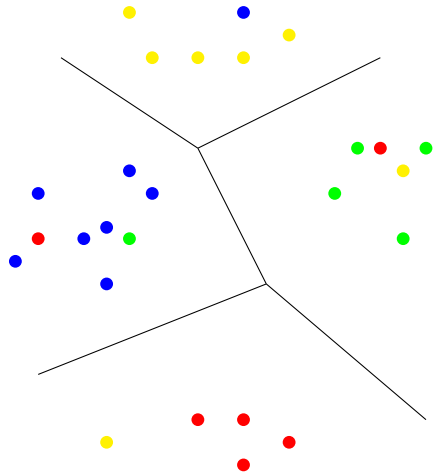
\begin{figure}
\begin{center}
\begin{tikzpicture}[scale=0.6]
\useasboundingbox (-5,-8) rectangle (6,5);
\fill [blue] (-1.5,0.5) circle (4pt);
\fill [green] (-1.5,-1) circle (4pt);
\fill [blue] (-3.5,0) circle (4pt);
\fill [blue] (-2.5,-1) circle (4pt);
\fill [blue] (-2,-2) circle (4pt);
\fill [red] (-3.5,-1) circle (4pt);
\fill [blue] (-4,-1.5) circle (4pt);
\fill [blue] (-1,0) circle (4pt);
\fill [blue] (-2,-0.75) circle (4pt);

\fill [green] (4.5,-1) circle (4pt);
\fill [green] (5,1) circle (4pt);
\fill [red] (4,1) circle (4pt);
\fill [green] (3.5,1) circle (4pt);
\fill [yellow] (4.5,0.5) circle (4pt);
\fill [green] (3,0) circle (4pt);

\fill [red] (0,-5) circle (4pt);
\fill [red] (1,-6) circle (4pt);
\fill [red] (2,-5.5) circle (4pt);
\fill [red] (1,-5) circle (4pt);
\fill [yellow] (-2,-5.5) circle (4pt);

\fill [yellow] (0,3) circle (4pt);
\fill [blue] (1,4) circle (4pt);
\fill [yellow] (2,3.5) circle (4pt);
\fill [yellow] (1,3) circle (4pt);
\fill [yellow] (-1,3) circle (4pt);
\fill [yellow] (-1.5,4) circle (4pt);

\draw (-3,3) -- (0,1);
\draw (4,3) -- (0,1);
\draw (0,1) -- (1.5,-2);
\draw (1.5,-2) -- (-3.5,-4);
\draw (1.5,-2) -- (5, -5);

\end{tikzpicture}
  \end{center}
\caption{For the clustering of Figure \ref{fig:ProblemSituation}, the power diagram of Figure \ref{fig:PDexampleWithSites} is a soft power diagram with six points outside of their clusters' respective cells.}
\label{fig:SoftPD}
\end{figure}

We are interested in applications that use fixed sites for the construction of such a power diagram, but we will introduce our model for free sites, as we are able to prove our results for this more general setting. Then we transfer the results to fixed sites, for which the programs we derive are efficiently solvable to global optimality by linear programming. For free sites, the use of a shared margin among cluster pairs makes the model hard to solve globally, but reasonable to solve locally -- which essentially is the best one can hope for.

The local optima of our model exhibit a special {\bf point counting property}, which is our main result denoted in Theorems \ref{thm:NuProperties} and \ref{thm:NuProperties2}: We provide programs $(P_{\text{MME}})$ and $(P_{\text{MEP}})$ that use a prescribed parameter $t\in \mathbb{N}$. A locally optimal solution then is a power diagram with a (not-necessarily positive) margin such that the number of margin errors in the power diagram is at most $t$. Further, the number of support vectors is at least $t+1$.

This gives rises to efficient algorithms for outlier detection and the computation of what we call {\bf least-squares thresholds}. The latter correspond to the minimal fraction of points which are margin errors in the construction of a soft power diagram with non-negative margin. These thresholds tell us how similar a given clustering is to a balanced least-squares assignment for a given set of sites.

We begin Chapter $2$ with some basic notation and necessary terminology to use power diagrams as classifiers. We use a direct, geometric approach to obtain a model for the construction of a separating power diagram of maximal margin when the underlying clustering is a balanced least-squares assignment. As a byproduct, we obtain an efficient test for the existence of a separating power diagram for a given clustering for any free sites (Lemma \ref{lem:25}). 

Chapter $3$ contains our main results: We construct two versions of our model for the derivation of a soft power diagram, and prove the point counting properties of the locally optimal solutions for these programs. We also show that these properties transfer to the case of fixed sites, for which we obtain efficient models based on linear programming.

In Chapter $4$, we use the point counting properties of our algorithms for two applications with fixed sites: First, we describe an approach to outlier detection (Algorithm SPD-OD, Section $4.1$). Then we discuss Algorithm LS-SPD for the computation of least-squares thresholds (Section $4.2$). We prove that this is efficiently possible by solving a number of linear programs that is logarithmic in the theoretically maximal number of margin errors (Theorem \ref{thm:last}). Then we present some empirical results that validate that these thresholds are intuitive apriori measures for the quality of a soft power diagram as a classifier. We conclude our discussion with a proof of concept for the solvability of our model for free sites to local optimality.

\section{Separating power diagrams}
We begin with some necessary terminology. Let throughout the present paper $k\geq 2$, $n\geq 2$ and $d\geq 1$ with $k,n,d \in \mathbb{N}$. Let $X:=\{x_1,\dots,x_n\}\subset \mathbb{R}^d$ be a set of vectors ({\bf point set} or {\bf data set}) in the $d$-dimensional Euclidean space. Further, let $S:=\{s_1,\dots,s_k\}$ with $s_i \neq s_j$ for $i\neq j$ be a set of $k$ distinct {\bf sites} in $\mathbb{R}^d$. 

We now define some basic notation for the discussion of a (partitioning) clustering: A tuple $\mathcal{C}:=(C_1,\dots,C_k)$ is a {\bf $k$-clustering} of $X$ if and only if $$C_i \subset X \text{ for all } i\leq k,\quad C_i \cap C_j = \emptyset \text{ for } i \neq j, \text{ and }\bigcup\limits_{i=1}^k C_i = X.$$ $C_i$ is the {\bf $i$-th cluster} of $\mathcal{C}$.

A $k$-clustering of $X$ is defined by an $X,k$-label function $c:\{1,\dots,n\} \rightarrow \{1,\dots,k\}$, which satisfies
$c(l)=i \Leftrightarrow x_l\in C_i.$
We use $|C_i|$ to refer to the number of points in cluster $C_i$. The tuple $|\mathcal{C}|:=(|C_1|,\dots,|C_k|)$ is the {\bf shape} of $\mathcal{C}$. Finally, we denote the arithmetic mean of the points in cluster $C_i$ by $c_i$. Throughout this paper, we assume that $c_i\neq c_j$ for $i\neq j$, which is natural when wanting to derive a classifier distinguishing these clusters.

Typically, $X$ and $k$ will be clear from the context. We then use the simpler wording clustering $\mathcal{C}$, and label function or simply label. Often, each of the clusters $C_i$ corresponds to a site $s_i$ of the same index. We then call the points $x\in C_i$  associated to $s_i$.

In our geometric approach, hyperplanes, halfspaces, and the interior of halfspaces are very important. We denote the interior of a (convex, continuous) set $P\subset \mathbb{R}^d$ as $\text{int}(P)$. A hyperplane $H_{a,\gamma}$ in $\mathbb{R}^d$ is the set $H_{a,\gamma}:=\{x\in \mathbb{R}^d: a^Tx=\gamma\}.$

Such a hyperplane $H_{a,\gamma}$ separates two sets $P_1, P_2 \subset \mathbb{R}^d$  in $\mathbb{R}^d$ (or is a separating hyperplane for $P_1, P_2$) if and only if
$P_1\subset  H^{\leq}_{a,\gamma}$ and $P_2\subset H^{\geq}_{a,\gamma}$. If $P_1\subset  \text{int}( H^{\leq}_{s,\gamma})$ and $P_2\subset \text{int}(H^{\geq}_{a,\gamma})$, then $H_{a,\gamma}$  separates $P_1$ and $P_2$ strictly (or is a strictly-separating hyperplane for $P_1, P_2$).

\subsection{Power diagrams}
We are now ready to turn to polyhedral cell complexes in $\mathbb{R}^d$. A well known special kind of these are Voronoi diagrams \cite{ak-99}. They appear in many applications and algorithms such as the classical $k$-means algorithm. 

A natural and powerful generalization of Voronoi diagrams are the so-called {\bf power diagrams} \cite{a-87}. The cell $P_i$ of such a power diagram is defined by a site $s_i\in \mathbb{R}^d$ and a real weight $w_i \geq 0$. It consists of all the points $x\in \mathbb{R}^d$ which are 'closest' to the site, where this distance is measured by the so-called {\bf power function} $$p_i(x):=\|s_i-x\|^2 -w_i.$$
Informally, the power function is the distance of $x$ to the closest point on a sphere of radius $\sqrt{w_i}$ around site $s_i$. We notate the set of weights as $\omega:=(w_1,\dots,w_k)^T$ and obtain the following formal definition.

\begin{definition}[$(S,\omega)$-power diagram]
\\An {\bf $(S,\omega)$-power diagram} is a decomposition $P:=(P_1,\dots,P_k)$ of $\mathbb{R}^d$ with $$P_i:=\{x \in \mathbb{R}^d:
\|s_i-x\|^2 -w_i \leq \|s_j-x\|^2-w_j  \quad\text{ for all } j \in \{1,\dots,k\} \backslash \{i\}\}$$
if $\text{dim} (P_i)=d$ for all $i \in \{1,\dots,k\}$.
\end{definition}

By $P$ being a decomposition of $\mathbb{R}^d$, we have $\text{int}(P_i)\cap \text{int}(P_j)=\emptyset$ for all $i\neq j$. This property combined with $\text{dim} (P_i)=d$ implies $s_i \neq s_j$ for $i\neq j$.

Before we turn to applications of these diagrams in data analysis tasks, let us provide a new, alternate representation, which will prove helpful in our analysis. We use the notation $\gamma:=(\gamma_1,\dots,\gamma_k)^T$ with $\gamma_i\in \mathbb{R}$ for all $i\leq k$ throughout the paper.

\begin{definition}[$(S,\gamma)$-power diagram]\label{def:newPD}
\\An {\bf $(S,\gamma)$-power diagram} is a decomposition $P:=(P_1,\dots,P_k)$ of $\mathbb{R}^d$ with $$P_i:=\{x \in \mathbb{R}^d: (s_j-s_i)^Tx\leq \gamma_j-\gamma_i \quad \text{ for all } j \in \{1,\dots,k\} \backslash \{i\}\}$$
if $\text{dim} (P_i)=d$ for all $i \in \{1,\dots,k\}$.
\end{definition}
Let us briefly confirm that the above definition of a power diagram is equivalent to the classical one. $P_i$ and $P_j$ are separated by the hyperplane \begin{eqnarray*}
H_{ij}& := &\{x \in \mathbb{R}^d: \|s_i-x\|^2 -w_i = \|s_j-x\|^2-w_j\} \\
& = & \{x \in \mathbb{R}^d: 2(s_j-s_i)^Tx = (s_j^Ts_j-w_j)-(s_i^Ts_i-w_i)\}
\end{eqnarray*}
Choosing$\gamma_i=\frac{1}{2}(s_i^Ts_i-w_i)$ for $i \leq k$ yields this new representation. We noted a strong connection of Definition \ref{def:newPD} to piecewise-linear separability \cite{bm-92}. A key difference lies in using weak inequalities instead of strict ones. This will help us in deriving a new program that is also able of construct a separating power diagram with points on the boundary of cells - sometimes these are the only separating power diagrams for a clustering. They cannot be found by the standard models in the literature.

When only the sites $S$, and not the specific representation of a power diagram, either by $\omega$ or $\gamma$, are of interest, we use the shorter notion $S$-power diagram. When we concern ourselves with applications where the sites are fixed or known from the context, we simply talk about a power diagram.

\subsection{Separating power diagrams}

A classical way to classify data is to find separating hyperplanes in between clusters. Let us define clusterings for which the hyperplanes of a power diagram have this separation property.

\begin{definition}[Separating power diagram]\label{def:spd}
\\Let $\mathcal{C}$ be a clustering of $X$ and let $P$ be a power diagram.
Then $\mathcal{C}$ allows the {\bf separating power diagram $P$} if $C_i\subset P_i$ for all $i\in \{1,\dots,k\}$ and $C_i\not\subset P_j$  for all $i \neq j$.
\end{definition}
We also say that $P$ is a separating power diagram for $\mathcal{C}$. Informally, for all $i\in \{1,\dots,k\}$, all points of cluster $C_i$ lie in $P_i$. Note that points may also lie on the boundary of the cell, i.e. $x\in P_i \not\Rightarrow x\in C_i$. On the other hand, the condition $C_i\not\subset P_j$  for all $i \neq j$ implies that not all points of $C_i$ lie on the common boundary of cells $P_i$ and $P_j$. If that was the case, $C_i$ would be fully contained in the separating hyperplane, invalidating the interpretation of the power diagram as a classifier.

If there are no points on the boundaries of the cells, all of the hyperplanes separate the corresponding clusters strictly. In this case, we talk about a strictly-separating power diagram. Figure \ref{fig:PDexampleWithSites} depicts an example.

Let us emphasize the strength of the separation property described in Definition \ref{def:spd}: It guarantees a lot more than the existence of a separating hyperplane for each pair of clusters, which can be demonstrated by explicitly constructing a clustering which allows pairwise separability of the clusters, but no separating power diagram. See \cite{b-10} for a provably minimal example. 

In fact, the property is tied to very special clusterings of point sets.  It is well-known that so-called least-squares assignments allow the construction of Voronoi diagrams such that each cluster lies in its own cell. The existence of a separating power diagram corresponds to the clustering being a {\bf balanced least-squares assignment.}

\begin{definition}[Balanced $(S,|\mathcal{C}|)$-least-squares assignment]
\\A clustering $\mathcal{C}$ is a balanced $(S,|\mathcal{C}|)$-least-squares assignment of $X$ if and only if \[\sum\limits_{i=1}^k \sum\limits_{x\in C_i}
\|s_i-x\|^2\] is minimal for all clusterings of $X$ of the same shape $|\mathcal{C}|$.
\end{definition}
The 'balanced' term refers to the minimality of the clustering with respect to all clusterings of the same shape, not all clusterings in general. Recall the program in Section $1.1$. Trivially its optimal clustering is a balanced least-squares assignment with respect to the cluster sizes of an optimal solution. We denote this 'induction' of cluster sizes as an $(S, |\mathcal{C}|)$-least-squares-assignment. Let $K:=(\kappa_1,\dots,\kappa_k)$ with $\kappa_i\in \mathbb{N}$ for all $i\leq k$. If we prescribe the cluster sizes to take values $|C_i|=\kappa_i$, we use the term $(S, K)$-least-squares-assignment.

In our notation, these least-squares assignments are connected to power diagrams in the following way \cite{aha-98}.

\begin{proposition}\label{thm:LSAandPD}
Let $X,S \subset \mathbb{R}^d$.
\begin{enumerate}
\item Let $K\subset \mathbb{N}^k$. Then there is an $(S,K)$-least-squares assignment $\mathcal{C}$ of $X$, and this $\mathcal{C}$ allows a separating $S$-power diagram.
\item If there is a separating $S$-power diagram $P$ for a clustering $\mathcal{C}$ of $X$, then $\mathcal{C}$ is an $(S,|\mathcal{C}|)$-least-squares assignment of $X$.
\end{enumerate}
\end{proposition}

\subsection{Maximum-margin power diagrams}
We now introduce a new model for the construction of a separating power diagram such that the minimal margin in between clusters is maximized. For fixed sites $S$, the model encompasses an efficient way of computing such a 'maximum-margin' separating $S$-power diagram corresponding to an $S$-least-squares assignment, as well as an efficient test for the existence of a separating power diagram (for any set of sites) for a given clustering. 

We start by recalling the necessary and sufficient conditions for $S$ and $\gamma$ yielding a separating power diagram for a set. Let $\mathcal{C}$ be a clustering of $X$. Then any $(S,\gamma)$ which satisfies
$$
\begin{array}{lrclcl}
  &(s_j-s_i)^Tx_l & \leq & \gamma_j-\gamma_i &  & (l\leq n,  j \leq k: j\neq i:=c(l))\\
& s_i & \neq & s_j & &  (i < j \leq k)
\end{array}
$$
yields a separating $(S,\gamma)$-power diagram. The given conditions guarantee that $x_l \in C_i$ is on the correct side of the hyperplane with normal $s_j-s_i\neq 0$ separating $P_i$ and $P_j$. The conditions $s_i \neq s_j$  can be replaced by linear constraints:

\begin{lemma}\label{lem:25}
Let $\mathcal{C}$ be a clustering of $X$. Then $\mathcal{C}$ allows a separating power diagram if and only if
$$
\begin{array}{lrclcl}
  &(s_j-s_i)^Tx_l & \leq & \gamma_j-\gamma_i &  & (l\leq n,  j \leq k: j\neq i:=c(l))\\
& (s_j-s_i)^T(c_j-c_i) & \geq & 1 & &  (i < j \leq k)
\end{array}
$$
has a feasible solution.
\end{lemma}

\begin{proof}
Recall that the $c_i$ refer to the arithmetic means of the clusters, and that they satisfy $c_i\neq c_j$ for $i\neq j$. $c_i$ lies in the convex hull of the points in $C_i$, and thus the first type of constraints implies $(s_j-s_i)^Tc_i  \leq  \gamma_j-\gamma_i \leq (s_j-s_i)^Tc_j$, which yields $(s_j-s_i)^T(c_j-c_i) \geq 0$. If $(s_j-s_i)^T(c_j-c_i) = 0$, the only case in which the first type of constraints is satisfied is when all points in $C_i$ and $C_j$ lie in the separating hyperplane of cells $P_i$ and $P_j$; but this is a contradiction to $C_i \not\subset P_j$.

Thus there is a $\delta$ with $(s_j-s_i)^T(c_j-c_i) \geq \delta > 0$ for all $i \neq j$. The cells of a power diagram are invariant under uniform scaling and translation of both the sites and $\gamma$-values, see \cite{aha-98,b-10}: The cells of a $(S,\gamma)$-power diagram are the same as those of a $(\frac{1}{\delta}\cdot S, \frac{1}{\delta} \cdot \gamma)$-power diagram, and the sites of the latter satisfy $(s_j-s_i)^T(c_j-c_i) \geq 1$. This proves the claim. \qed
\end{proof}

Lemma \ref{lem:25} tells us that one is able to efficiently find a separating power diagram for a given clustering if one exists, or decide that there is none. In particular, it is possible to construct this power diagram with points on the boundary of the cells (in contrast to the common state-of-the-art alltogether multiclass support vector machines models transferred to the context of hard separation). This is helpful when there is no strictly-separating power diagram for the clustering at hand.

If the sites $S$ are fixed, the second set of conditions is not necessary. The normals of the hyperplanes are then collinear to $s_j-s_i$ for clusters $C_i$ and $C_j$. For these directions, we only have to compute the variables $\gamma$ that determine the hyperplane positioning. 

In this case, one may actually preprocess the values $(s_j-s_i)^Tx_l$ to obtain constants $\sigma_{ij}'= \max \limits_{x\in C_i} (s_j-s_i)^Tx$, for which we can denote the system as a much smaller linear program:
\begin{corollary}\label{cor:26}
Let $\mathcal{C}$ be a clustering of $X$, and let $S$ be given. Then $\mathcal{C}$ allows a separating $S$-power diagram if and only if
$$
\begin{array}{lrclcl}
  &\sigma_{ij}'  & \leq & \gamma_j-\gamma_i &  & (i \leq k, j \leq k: i\neq j)\\
\end{array}
$$
has a feasible solution.
\end{corollary}

We are particularly interested in clusterings of large {\bf margin $\epsilon$}, i.e. a large minimal Euclidean distance of a point $x_l\in C_i$ to the boundary of its cell. This corresponds to a large minimal distance to one of the hyperplanes $H_{ij}$ defining the cell. To obtain this geometric interpretation, the variables and parameters in the above constraints have to refer to geometric distances, which we obtain by dividing both $s_j-s_i$ and $\gamma_j-\gamma_i$ by $\| s_j-s_i \|$ for all $i\neq j$. This is a crucial step to obtaining the point counting properties of the program that we prove in the next chapter. For fixed sites, the program remains linear, but for free sites we obtain a nonlinear set of constraints. 

Let us provide a formal definition for the power diagrams which satisfy these constraints, and for the margin of a separating power diagram. Let here, and throughout the remainder of this paper $s_{ij}:= \frac{s_j-s_i}{\| s_j-s_i \|}$, $\gamma_{ij}:= \frac{\gamma_j-\gamma_i}{\| s_j-s_i \|}$

\begin{definition}[$\epsilon$-margin separating power diagram]\label{def:eSPD}
Let $\mathcal{C}$ be a clustering of $X$, and let $P$ be an $S$-power diagram. $P$ is an {\bf $\epsilon$-margin separating power diagram} for $\mathcal{C}$ if and only if $$
\begin{array}{lrclcl}
  & s_{ij}^Tx_l + \epsilon & \leq & \gamma_{ij} &  & (l\leq n,  j \leq k: j\neq i:=c(l))\\
& (s_j-s_i)^T(c_j-c_i) & \geq & 1 & &  (i < j \leq k)
\end{array}
$$
$\epsilon$ then is the {\bf margin} of $P$.
\end{definition}
Note that we do not ask for $\epsilon \geq 0$ in the above definition. Informally, a margin of $\epsilon < 0$ means that it is 'ok' for a point to be on the wrong side of a separating hyperplane, but not more than Euclidean distance $\epsilon$ from it.  Fur further justification of such an approach see e.g. \cite{pwhb-03}.

Complementing the constraints in Lemma \ref{lem:25} and Definition \ref{def:eSPD} by an objective function maximizing the margin $\epsilon$, we obtain the following program:
$$
\begin{array}{lrclcl}
(P_{\text{SPD}}) &  \multicolumn{3}{c}{\max\, \epsilon}        &  &     \\
  & s_{ij}^Tx_l + \epsilon & \leq & \gamma_{ij} &  & (l\leq n,  j \leq k: j\neq i:=c(l))\\
& (s_j-s_i)^T(c_j-c_i) & \geq & 1 & &  (i < j \leq k)
\end{array}
$$
The constraints of type $s_{ij}^Tx_l + \epsilon  \leq  \gamma_{ij}$ are non-convex for free sites. Conversely, for fixed sites, the above is a linear program. We call the power diagram which is optimal for this program the {\bf maximum-margin power diagram} of $\mathcal{C}$. Let us sum up our construction.

\begin{theorem}\label{thm:SPCcomp}
Let $\mathcal{C}$ be a clustering of $X$. Then a maximum-margin separating power diagram of X corresponds to a global optimum of $(P_{\text{SPD}})$
\end{theorem}

For fixed sites, we only have to keep the first set of linear inequalities. Using $\sigma_{ij}:= \max \limits_{x\in C_i} s_{ij}^Tx$, we obtain the linear program 
$$
\begin{array}{lrclcl}
(P_{\text{SPD}}') & \multicolumn{3}{c}{\max\, \epsilon}        &  &     \\
&  \sigma_{ij}  + \epsilon & \leq & \gamma_{ij} & &  (i \leq k, j \leq k: i\neq j).
\end{array}
$$
and 
\begin{corollary}\label{cor:SPCcomp}
Let $\mathcal{C}$ be a clustering of $X$, and let $S$ be given.   Then a maximum-margin separating $S$-power diagram of X corresponds to an optimum of the linear program $(P_{\text{SPD}}')$.
\end{corollary}

W.l.o.g. $\gamma_1=0$, then the linear program has $k-1$ variables and $k\cdot (k-1)$ constraints.

\section{Soft power diagrams}
In this chapter, we extend programs $(P_{\text{SPD}})$ and $(P_{\text{SPD}}')$ by a soft separation scheme. Like the $\nu$-soft-margin support vector machine for two clusters \cite{sswb-00}, our model allows the use of a parameter to prescribe an upper bound on the number of misclassified points. Note that the definition of a misclassified point has to be more involved in our multiclass scenario. As a service to the reader, we begin by formally defining two intuitive concepts in our notation in the next section.

In the succeeding section, we then prove the desired point counting properties, and show that they transfer to the case of fixed sites. Even though our programs are very different from the ones constructed for binary $\nu$-soft separation, investigating the partial derivatives of the underlying Lagrange function will again prove helpful in deriving these properties. For fixed sites, we can find a power diagram adhering to the properties by linear programming.

\subsection{Multiclass support vectors and margin errors}
Recall program $(P_{\text{SPD}})$. For a given power diagram $P$ and a margin $\epsilon$, generally some of the points of a clustering $\mathcal{C}$ do not satisfy the constraints $s_{ij}^Tx + \epsilon \leq \gamma_{ij}$. These are the misclassified points, i.e. the margin errors of the power diagram.

In this section, we formally define two types of margin errors and support vectors in a multiclass scenario. Both have their own advantages and applications. We then describe the separation properties of our soft power diagram for both of these interpretations in Sections $3.2$ and $3.3$. We begin with the case where we consider the classification of a point with respect to each separating hyperplane of the power diagram.

\begin{definition}[Multiclass support vector and margin error]
Let $\mathcal{C}$ be a $k$-clustering of $X$, let $P$ be an $\epsilon$-margin $(S,\gamma)$-power diagram. Then $x\in C_i$ is a {\bf multiclass support vector with respect to $C_j$} if and only if $s_{ij}^Tx + \epsilon \geq \gamma_{ij}$, and a {\bf multiclass margin error with respect to $C_j$} if and only if $s_{ij}^Tx + \epsilon > \gamma_{ij}$.
\end{definition}
We leave out the 'multiclass' part of the wording if the context is clear. By this definition, a single point $x\in C_i$ can be a support vector, respectively margin error, with respect to multiple other clusters $C_j$. Analogously, $x\in C_i$ is a support vector with respect to $P_j$ or the hyperplane in between $P_i$ and $P_j$. When counting the number of support vectors and margin errors (as we do in the next section), such a point counts as $t$ support vectors or margin errors, where $t$ is the number of clusters $C_j$ that it is a multiclass margin error or support vector to. We relate to such a point as a {\bf $t$-fold margin error} (or {\bf support vector)}.

In some applications this notion of $t$-fold margin support vectors or margin errors, which allows points to correspond to $t > 1$ margin errors,  is less desirable than another one: In outlier detection, one might only be interested in which points correspond to at least one multiclass margin error, and does not care about whether these points are badly positioned with respect to one or many hyperplanes. In this case, we use the following definition:
\begin{definition}[Support vector point and margin error point]
Let $\mathcal{C}$ be a $k$-clustering of $X$, let $P$ be an $\epsilon$-margin $(S,\gamma)$-power diagram. Then $x\in C_i$ is a {\bf  support vector point} if and only if there is a $j\neq i$ such that $s_{ij}^Tx + \epsilon \geq \gamma_{ij}$, and a {\bf margin error point} if and only if there is a $j\neq i$ such that $s_{ij}^Tx + \epsilon > \gamma_{ij}$.
\end{definition}

Let us close this section with an example: Consider Figure \ref{fig:MulticlassSM}, which depicts the clustering and soft power diagram of Figure \ref{fig:SoftPD}, along with a margin. The red margin error point on the right-hand side is a multiclass margin error with respect to both the hyperplane separating the red and green, and the hyperplane separating the red and blue clusters; it is (at least) a double multiclass margin error.

\begin{figure}
\begin{center}
\begin{tikzpicture}[scale=0.6]
\useasboundingbox (-5,-8) rectangle (6,5);
\draw[line width=1.5cm, color=gray] (-3,3) -- (0,1);
\draw[line width=1.5cm, color=gray]  (4,3) -- (0,1);
\draw[line width=1.5cm, color=gray]  (0,1) -- (1.5,-2);
\draw[line width=1.5cm, color=gray]  (1.5,-2) -- (-3.5,-4);
\draw[line width=1.5cm, color=gray]  (1.5,-2) -- (5, -5);
\draw (-3,3) -- (0,1);
\draw (4,3) -- (0,1);
\draw (0,1) -- (1.5,-2);
\draw (1.5,-2) -- (-3.5,-4);
\draw (1.5,-2) -- (5, -5);

\draw [dashed] (1.5,-2) -- (-3.75, 2.5);
\draw [dashed] (1.5,-2) -- (5, -0.6);

\fill [blue] (-1.5,0.5) circle (4pt);
\fill [green] (-1.5,-1) circle (7pt);
\fill [blue] (-3.5,0) circle (4pt);
\fill [blue] (-2.5,-1) circle (4pt);
\fill [blue] (-2,-2) circle (4pt);
\fill [red] (-3.5,-1) circle (7pt);
\fill [blue] (-4,-1.5) circle (4pt);
\fill [blue] (-1,0) circle (4pt);
\fill [blue] (-2,-0.75) circle (4pt);

\fill [green] (4.5,-1) circle (4pt);
\fill [green] (5,1) circle (4pt);
\fill [red] (4,1) circle (7pt);
\fill [green] (3.5,1) circle (4pt);
\fill [yellow] (4.5,0.5) circle (7pt);
\fill [green] (3,0) circle (4pt);

\fill [red] (0,-5) circle (4pt);
\fill [red] (1,-6) circle (4pt);
\fill [red] (2,-5.5) circle (4pt);
\fill [red] (1,-5) circle (4pt);
\fill [yellow] (-2,-5.5) circle (7pt);

\fill [yellow] (0,3) circle (4pt);
\fill [blue] (1,4) circle (7pt);
\fill [yellow] (2,3.5) circle (4pt);
\fill [yellow] (1,3) circle (4pt);
\fill [yellow] (-1,3) circle (7pt);
\fill [yellow] (-1.5,4) circle (4pt);
\end{tikzpicture}
  \end{center}
\caption{An $\epsilon$-margin soft power diagram for four clusters in $\mathbb{R}^2$. The large dots are the support vector points (in fact, all of them are margin error points). The red dot on the right-hand size is (at least) a double multiclass margin error.}
\label{fig:MulticlassSM}
\end{figure}
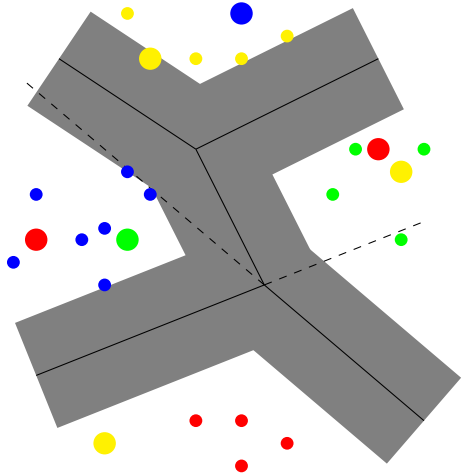

Next, we present how to compute a soft power diagram with a prescribed upper bound on the number of multiclass margin errors. In the succeeding section, we do the same for margin error points.

\subsection{Bounding the number of multiclass margin errors}
We now extend $(P_{\text{SPD}})$ to be able to prescribe an upper bound on the number of multiclass margin errors: Let us introduce variables $\xi_{jl}\geq 0$ for all $l\leq n$ and $j \leq k$ with $j\neq i:=c(l)$.  For a given $\epsilon$-margin $(S,\gamma)$-power diagram, the constraints
$$s_{ij}^Tx_l  +\epsilon  \; \leq \;    \gamma_{ij} +\xi_{jl} \quad (l \leq n, j\leq k: j\neq i:=c(l))$$
can then always be satisfied by choosing $\xi_{jl}>0$ sufficiently large.  Note that, if we choose $\xi_{jl}$ minimal among all feasible values, each component $\xi_{jl}>0$ corresponds to a margin error of point $x_l\in C_i$ with respect to $C_j$. Recall that a single point $x_l$ may yield multiple margin errors, by this definition. Note further that the support vectors which are not margin errors satisfy 
$$s_{ij}^Tx_l  +\epsilon  \; = \;    \gamma_{ij} \quad (l \leq n, j\leq k: j\neq i:=c(l)).$$

We now formally define this construction.  For the sake of a simple notation, let $\xi:=(\xi_{jl})\in \mathbb{R}^{(k-1)n}$ be a vector listing all coefficients $\xi_{jl}$ for $x_l\in C_i$ and $j\neq i:=c(l)$.

\begin{definition}[(Multiclass) soft power diagram]\label{thm:MSPD}
Let $\mathcal{C}$ be a clustering of $X$, and let $S:=\{s_1,\dots,s_k\}\subset \mathbb{R}^d$, $\gamma \in \mathbb{R}^k$, $\epsilon\in \mathbb{R}$ and $\xi:=(\xi_{jl})\geq 0 \in \mathbb{R}^{(k-1)n}$. Finally, let $P$ be an $\epsilon$-margin $(S,\gamma)$-power diagram.
$P$ is an $\epsilon$-margin {\bf (multiclass) soft power diagram} for $\mathcal{C}$ if and only if $$s_{ij}^Tx_l  +\epsilon   \; \leq  \;   \gamma_{ij} +\xi_{jl} \quad  (l \leq n, j\leq k: j\neq i:=c(l)).$$
\end{definition} 
We refer to a soft power diagram by its parameters $(S,\gamma,\epsilon,\xi)$. By the above, we may assume that the $\xi_{jl}$ are all minimal for fixed $(S,\gamma,\epsilon)$, such that $s_{ij}^Tx_l  +\epsilon   \leq   \gamma_{ij} +\xi_{jl}$ is satisfied. Further, if the sites in $S$ are fixed or clear from the context, we refer to the power diagram by the tuple $(\gamma,\epsilon)$.

 To find a 'good' soft power diagram, we have to add a penalty term in the objective function for the margin error values  $\xi_{jl}>0$. For this purpose, we use a parameter $t\in \mathbb{N}$, and add it in the following form
$$
\begin{array}{lrclcl}
(P_{\text{MME}}) & \multicolumn{5}{c}{ \max \quad \Theta_{\text{MME}}(t):=\epsilon - \frac{t+\frac{1}{2}}{t(t+1)} \sum\limits_{l=1}^n \sum\limits_{j\neq c(l)}  \xi_{jl}}   \\
&  s_{ij}^Tx_l  +\epsilon & \leq &   \gamma_{ij} +\xi_{jl} & & (l \leq n, j\leq k: j\neq i:=c(l))\\
& (s_j-s_i)^T(c_j-c_i) & \geq & 1 & &  (i < j \leq k)\\
& \xi_{jl} & \geq & 0 & & (l \leq n, j\leq k: j\neq c(l))
\end{array}
$$
Choosing the penalty on the $\xi_{jl}$ in this way allows us to directly bound the number of margin errors by $t$ in our construction. This is the first version of our main result.
\begin{theorem}\label{thm:NuProperties}
Let $\mathcal{C}$ be a clustering of $X$, and let $t\in \mathbb{N}$. Let further $(S^*, \gamma^*,\epsilon^*, \xi^*)$ be a local optimum of $(P_{\text{MME}})$. Then
 $(S^*,  \gamma^*, \epsilon^*, \xi^*)$ yields a soft power diagram $P$ of maximal margin $\epsilon^*$ for fixed $S^*,\gamma^*$ such that  $t$ is an upper bound on the number of multiclass margin errors for $P$, and $t+1$ is a lower bound on the number of multiclass support vectors for $P$.
\end{theorem}

\begin{proof} First, note that  $ \min_{i\neq j} \|s_j-s_i\| \geq \lambda > 0$ due to  $(s_j-s_i)^T(c_j-c_i) \geq  1$. Thus the feasibility region of $(P_{\text{MME}})$ is a closed set. Then there exist  local optima of $(P_{\text{MME}})$, and it is reasonable to discuss properties of these local optima.

Any feasible $(S,\gamma, \epsilon, \xi)$ for $(P_{\text{MME}})$ defines a soft power diagram, by Definition \ref{thm:MSPD}. Let now $(S^*,  \gamma^*, \epsilon^*, \xi^*)$ correspond to a local optimum, and let $L_{\text{MME}}$ refer to the Lagrange function of $(P_{\text{MME}})$.

By using the Lagrange multipliers $\alpha:=(\alpha_{jl}), \beta:=(\beta_{ij}), \delta:=(\delta_{jl})\geq 0$ for to the first, second, and third type of constraints of $(P_{\text{MME}})$ (and the corresponding values of $i,j$ and $l$), we can denote $L_{\text{MME}}=L_{\text{MME}}(S,  \gamma, \epsilon, \xi,\alpha, \beta, \delta)$ as
\begin{eqnarray*}
L_{\text{MME}} &= &\epsilon - \frac{t+\frac{1}{2}}{t(t+1)} \sum\limits_{l=1}^n \sum\limits_{j\neq c(l)} \xi_{jl}+\\
& + & \sum\limits_{l=1}^n \sum\limits_{j\neq c(l)} \alpha_{jl} \cdot ( \gamma_{ij}+ \xi_{jl} - s_{ij}^Tx_l -\epsilon) +\\
& + & \sum\limits_{i=1}^{k-1} \sum\limits_{j= i+1}^k \beta_{ij} \cdot ((s_j-s_i)^T(c_j-c_i)- 1 )+\\
& + & \sum\limits_{l=1}^n \sum\limits_{j\neq c(l)}  \delta_{jl} \cdot \xi_{jl}.
 \end{eqnarray*}
Note that $L_{\text{MME}}$ is differentiable for any $(S^*,  \gamma^*, \epsilon^*, \xi^*,\alpha, \beta, \delta)$. Let now $(\alpha^*, \beta^*, \delta^*)$ be optimal dual variables for $(S^*,  \gamma^*, \epsilon^*, \xi^*)$. 

$L_{\text{MME}}$ has a special saddle point at $(S^*,  \gamma^*, \epsilon^*, \xi^*, \alpha^*, \beta^*, \delta^*)$, and thus the primal partial derivatives of $L_{\text{MME}}$ are equal to $0$ at that point. Let us denote such a partial derivative of $L_{\text{MME}}$ with respect to $\epsilon$ or $\xi_{jl}$ as $\frac{\delta L}{\sigma \epsilon}$, respectively $\frac{\sigma L}{\sigma \xi_{jl}}$. Considering just these two derivatives, we obtain two sets of conditions:
\begin{eqnarray*}
\frac{\sigma L}{\sigma \epsilon}\;=0& =& 1-\sum\limits_{l=1}^n \sum\limits_{j\neq c(l)} \alpha_{jl}^*\\
\frac{\sigma L}{\sigma \xi_{jl}}\;=0& =& -\frac{t+\frac{1}{2}}{t(t+1)} + \alpha_{jl}^* + \delta_{jl}^*\quad \quad   (l \leq n, j\leq k: j\neq c(l))
 \end{eqnarray*}
Let us rewrite this as
\begin{eqnarray}
1 & =& \sum\limits_{l=1}^n \sum\limits_{j\neq c(l)} \alpha_{jl}^*\\
\frac{t+\frac{1}{2}}{t(t+1)} & =& \alpha_{jl}^* + \delta_{jl}^*\quad \quad   (l \leq n, j\leq k: j\neq c(l)).
 \end{eqnarray}
These conditions allow us to investigate the number of multiclass margin errors and support vectors. Note that $x_l\in C_i$ is a support vector for the hyperplane  in between cells $P_i$ and $P_j$ if $\alpha_{jl}^* > 0$.  A margin error corresponds directly to the case $\xi_{jl}^*>0$. Note further that 
$$\frac{1}{t}>\frac{t+\frac{1}{2}}{t(t+1)}=\frac{1}{2}(\frac{1}{t}+\frac{1}{t+1}) >\frac{1}{t+1}.$$
First, let $x_l$ be a support vector. As $\delta_{jl}^*\geq 0$, $(2)$ implies that $\frac{1}{t} > \frac{t+\frac{1}{2}}{t(t+1)}\geq \alpha_{jl}^*$. Thus, at least $t+1$ values of  $\alpha_{jl}^*$ have to be greater than zero, otherwise $(1)$ is not satisfied.

If $x_l$ is a margin error with respect to the hyperplane in between cells $P_i$ and $P_j$, the corresponding $\xi_{jl}^*$ satisfies $\xi_{jl}^*> 0$, which implies that $\delta_{jl}^*=0$. Then $(2)$ implies that the corresponding $\alpha_{jl}^*=\frac{t+\frac{1}{2}}{t(t+1)}>\frac{1}{t+1}$. Due to $(1)$, there are at most $t$ margin errors, otherwise $ \sum\limits_{l=1}^n \sum\limits_{j\neq c(l)}  \alpha_{jl}^* > 1$. 

It remains to prove that $\epsilon^*$ is maximal for the power diagram fixed by $S^*$ and $\gamma^*$, and for which there are at least $t+1$ support vectors and at most $t$ margin errors. By increasing the margin by an arbitrarily small amount, all of the former support vectors become margin errors, which conflicts with the upper bound of at most $t$ margin errors. This proves the claim.  \qed
\end{proof}
These counting properties transfer to the case of fixed sites, as we only used the partial derivatives with respect to $\epsilon$ and the $\xi_{jl}$ for the proof. We obtain the linear program 
$$
\begin{array}{lrclcl}
(P_{\text{MME}}') & \multicolumn{5}{c}{ \max \quad \Theta_{\text{MME}}(t):=\epsilon - \frac{t+\frac{1}{2}}{t(t+1)} \sum\limits_{l=1}^n \sum\limits_{j\neq c(l)}  \xi_{jl}}   \\
&  s_{ij}^Tx_l  +\epsilon & \leq &   \gamma_{ij} +\xi_{jl} & & \quad (l \leq n, j\leq k: j\neq i:=c(l))\\
& \xi_{jl} & \geq & 0 &&\quad  (l \leq n, j\leq k: j\neq c(l))
\end{array}
$$
and
\begin{corollary}\label{cor:NuProperties}
Let $\mathcal{C}$ be a clustering of $X$, let $S$ be given, and let $t\in \mathbb{N}$. Let further $( \gamma^*, \epsilon^*, \xi^*)$ be an optimal solution of $(P_{\text{MME}}')$. Then $( \gamma^*, \epsilon^*, \xi^*)$ yields a soft $S$-power diagram $P$ of maximal margin $\epsilon^*$ for fixed $\gamma^*$ such that $t$ is an upper bound on the number of multiclass margin errors for $P$,  and $t+1$ is a lower bound on the number of multiclass support vectors for $P$.
\end{corollary}

Next, we use a similar construction to obtain a program which provides an upper bound on the number of margin error points and a lower bound on the number of support vector points.

\subsection{Bounding the number of margin error points}
In many applications, we are interested in bounding the number of points that correspond to one or more margin errors. Recall $(P_{\text{MME}})$. For an optimal solution for the fixed $\gamma$ and $\epsilon$, each $\xi_{jl}$ is chosen minimally such that the constraints are feasible. For each $l \leq n$, there are $k-1$ values $\xi_{jl}$ for $j\neq i:=c(l)$. The idea to bounding the number of margin error points instead of multiclass margin errors is to use a variable $\xi_{l}$ which satisfies $\xi_{l}\geq \xi_{jl}$ for all $j\neq i$. Writing $\xi:=(\xi_1,\dots,\xi_n)^T$, we obtain $$s_{ij}^Tx_l  +\epsilon \;  \leq \;   \gamma_{ij} +\xi_{l}\quad  (l \leq n, j\leq k: j\neq i:=c(l)).$$ This yields another variant of our soft power diagrams.

\begin{definition}[(Point-based) soft power diagram]
Let $\mathcal{C}$ be a clustering of $X$, and let $S:=\{s_1,\dots,s_k\}\subset \mathbb{R}^d$, $\gamma \in \mathbb{R}^k$, $\epsilon\in \mathbb{R}$ and $\xi:=(\xi_l)\geq 0 \in \mathbb{R}^k$. Finally, let $P$ be an $\epsilon$-margin $(S,\gamma)$-power diagram.
 $P$ is an $\epsilon$-margin {\bf (point-based) soft power diagram} for $\mathcal{C}$ if and only if $$s_{ij}^Tx_l  +\epsilon \; \leq  \;    \gamma_{ij} +\xi_{l} \quad  (l \leq n, j\leq k: j\neq i:=c(l)).$$
\end{definition}
Using this definition, we can bound the number of margin error points and support vector points by using the following penalty term for the  variables $\xi_{l}$:
$$
\begin{array}{lrclcl}
(P_{\text{MEP}}) & \multicolumn{5}{c}{\max \quad \Theta_{\text{MEP}}(t):=\epsilon -  \frac{t+\frac{1}{2}}{t(t+1)} \sum\limits_{l=1}^n \xi_{l}}   \\
&  s_{ij}^Tx_l  +\epsilon & \leq &  \gamma_{ij} +\xi_{l} & & (l \leq n, j\leq k: j\neq i:=c(l))\\
& (s_j-s_i)^T(c_j-c_i) & \geq & 1 & &  (i < j \leq k)\\
& \xi_{l} & \geq & 0 & &  (l \leq n)
\end{array}
$$
This yields the second variant of our main result.
\begin{theorem}\label{thm:NuProperties2}
Let $\mathcal{C}$ be a clustering of $X$, and let $t\in \mathbb{N}$. Let further $(S^*, \gamma^*, \epsilon^*, \xi^*)$ be a local optimum of $(P_{\text{MEP}})$. Then $(S^*,  \gamma^*, \epsilon^*, \xi^*)$ yields a soft power diagram $P$ of maximal margin $\epsilon^*$ for fixed $S^*,\gamma^*$ such that  $t$ is an upper bound on the number of  margin error points for $P$,  and $t+1$ is a lower bound on the number of support vector points for $P$.
\end{theorem}
\begin{proof}
Analogously to the proof of Theorem \ref{thm:NuProperties}, we consider the Lagrange function $L_{\text{MEP}}$ of $(P_{\text{MEP}})$. Using the Lagrange multipliers $\alpha:=(\alpha_{jl}), \beta:=(\beta_{ij}), \delta:=(\delta_{l})\geq 0$, it takes the form
\begin{eqnarray*}
L_{\text{MEP}} & = & \epsilon - \frac{t+\frac{1}{2}}{t(t+1)} \sum\limits_{l=1}^n \xi_{l}+\\
& + & \sum\limits_{l=1}^n \sum\limits_{j\neq c(l)} \alpha_{jl} \cdot ( \gamma_{ij}+ \xi_{l} - s_{ij}^Tx_l -\epsilon) +\\
& + & \sum\limits_{i=1}^{k-1} \sum\limits_{j= i+1}^k \beta_{ij} \cdot ((s_j-s_i)^T(c_j-c_i)- 1 )+\\
& + & \sum\limits_{l=1}^n  \delta_{l} \cdot \xi_{l}.
 \end{eqnarray*}
Again, the partial derivatives of $L_{\text{MEP}}$ with respect to the primal variables at a saddle point $(S^*,  \gamma^*, \epsilon^*, \xi^*, \alpha^*, \beta^*, \delta^*)$ of optimal primal and dual variables  are equal to zero. We denote the partial derivatives with respect to $\epsilon$ or $\xi_l$ as $\frac{\sigma L}{\sigma \epsilon}$, respectively $\frac{\sigma L}{\sigma \xi_{l}}$, and obtain the set of conditions
\begin{eqnarray}
1 & =& \sum\limits_{l=1}^n \sum\limits_{j\neq c(l)}  \alpha_{jl}^*\\
 \frac{t+\frac{1}{2}}{t(t+1)} & =& \sum\limits_{j \neq c(l)} \alpha_{jl}^* + \delta_{l}^*\quad \quad    (l \leq n, j\leq k: j\neq c(l)).
 \end{eqnarray}
 Note that due to $\alpha^* \geq 0$, $x_l \in C_i$ is a support vector point if $\alpha_{jl}^*> 0$ for a $j\neq i$.  A margin error point corresponds directly to the case $\xi_{l}^*>0$.

By $\delta^* \geq 0$ and $(4)$, we know that $\sum\limits_{j\neq c(l)} \alpha_{jl}^* \leq  \frac{t+\frac{1}{2}}{t(t+1)}<  \frac{1}{t}$ for all $l\leq n$. By $(3)$, for at least $t+1$ points $x_l$, we have $\sum\limits_{j\neq c(l)} \alpha_{jl}^*  > 0$. For each of these points $x_l$, at least one of the values $\alpha_{jl}^*$ satisfies $\alpha_{jl}^*>0$. Thus, all of them are support vector points.

If $x_l$ is a margin error point, then $\xi_l^*>0$, which implies $\delta_l^*=0$. Then $\frac{1}{t} >  \frac{t+\frac{1}{2}}{t(t+1)} = \sum\limits_{j\neq c(l)} \alpha_{jl}^*$ by $(4)$, and then $(3)$ implies that at most $t$ points are margin error points.

The maximality of $\epsilon^*$ with respect to $\gamma^*$ and the given bounds follows by the same arguments as in the proof of Theorem \ref{thm:NuProperties}.  \qed
\end{proof}
Again, the counting properties transfer to the case of fixed sites. We obtain the linear program 
$$
\begin{array}{lrclcl}
(P_{\text{MEP}}') & \multicolumn{5}{c}{\max \quad \Theta_{\text{MEP}}(t):=\epsilon -  \frac{t+\frac{1}{2}}{t(t+1)} \sum\limits_{l=1}^n \xi_{l}}   \\
&  s_{ij}^Tx_l  +\epsilon & \leq &    \gamma_{ij} +\xi_{l} & & \quad (l \leq n, j\leq k: j\neq i:=c(l))\\
& \xi_{l} & \geq & 0 & & \quad (l \leq n)
\end{array}
$$
\begin{corollary}\label{cor:NuProperties2}
Let $\mathcal{C}$ be a clustering of $X$, let $S$ be given, and let $t\in \mathbb{N}$. Let further $( \gamma^*, \epsilon^*, \xi^*)$ be an optimal solution of program $(P_{\text{MEP}}')$. Then $( \gamma^*, \epsilon^*, \xi^*)$ yields a soft $S$-power diagram $P$ of maximal margin $\epsilon^*$ for fixed $\gamma^*$ such that $t$ is an upper bound on the number of  margin error points for $P$, and $t+1$ is a lower bound on the number of support vector points for $P$.
\end{corollary}

\section{Applications}
In this chapter, we exhibit two immediate applications of $(P_{\text{MME}}')$ and $(P_{\text{MEP}}')$ in outlier detection and for the computation of least-squares thresholds. They rely fundamentally on the ability to control the number of margin errors in the construction of a soft power diagram, as denoted in Corollaries \ref{cor:NuProperties} and \ref{cor:NuProperties2}. The feasibility regions of these programs are polytopes and they can be solved by linear programming.

In contrast, the programs $(P_{\text{SPD}})$, $(P_{\text{MME}})$ and $(P_{\text{MEP}})$ optimize linear objective functions over non-convex closed sets. While these programs are not at the heart of our target applications, as a proof of concept, we conclude our discussion with some computations of local optima with tools of nonlinear programming.

Our empiric results were derived using a standard laptop.\footnote[3]{The laptop uses Windows 7, 64 bit, 4 GB of RAM and an Intel Core i7-2630QM CPU at $2.00 \slash 2.00$ Ghz. The linear programs were solved using Xpress Optimizer (Version 3.2.2) by FICO. The nonlinear programs were solved to local optimality using  Ipopt \cite{wb-06}.} We tested our methods on about 20 data sets from the LIBSVM repository for multiclass classification \cite{cc-01} and here report on the representative results for the data sets {\em dna}, {\em vowel}, {\em satimage} and {\em shuttle}, with their provided training and testing sets.

{\em dna} consists of $1400$ training and $1186$ testing points of dimension $180$, partitioned into $3$ clusters, {\em vowel} consists of $528$ training and $462$ testing points of dimension $10$, partitioned into $11$ clusters, {\em satimage} consists of $3194$ training and $2000$ testing points of dimension $36$, partitioned into $6$ clusters and {\em shuttle} consists of $30450$ training and $14500$ testing points of dimension $9$, partitioned into $7$ clusters.

\subsection{Outlier detection}

The presented programs have a direct application in outlier detection: We start with a clustering and a set of representative sites for the clusters, and identify which points create margin errors when constructing a corresponding soft power diagram. These are the points that we consider outliers (or data noise) for their clusters. For this interpretation, both our point counting property and the fact that we use a shared margin for all pairs of clusters are helpful.

Given a  prescribed number $t$, both algorithms then compute a power diagram for the given set of sites and margin $\epsilon$ for which there are at most $t$ multiclass margin errors or $t$ margin error points. 

It is trivial to identify them from the optimal solution of the programs by the conditions $\xi_{jl}>0$ or $\xi_l>0$. We sum up this approach to outlier detection in Algorithm $1$, with respect to margin error points.

\begin{table}
\begin{itemize}
\item {\bf Input}: $d,k,n, t \in \mathbb{N}$,  clustering $\mathcal{C}$ of $X\subset\mathbb{R}^d$, sites $S\subset  \mathbb{R}^d$
\item {\bf Output:} soft $S$-power diagram $P$ defined by $(\gamma, \epsilon,\xi)$ and its set $M\subset X$ of  $|M|\leq t$ margin error points
\item {\bf 1.} Find the optimal solution $(\gamma, \epsilon,\xi)$ (defining soft power diagram $P$)  for the LP
\begin{eqnarray*}
\max \epsilon -  \frac{t+\frac{1}{2}}{t(t+1)} \sum\limits_{l=1}^n \xi_{l}&&\\
\text{ s.t. }\quad s_{ij}^Tx_l  +\epsilon  \leq   \gamma_{ij} +\xi_{l} &&\quad  (l \leq n, j\leq k: j\neq i:=c(l))\\
\xi_{l}  \geq  0 &&\quad (l \leq n)
 \end{eqnarray*}
\item {\bf 2.} Set $M:=\{x_l\in X: \xi_l >0\}$. Return $P$ and $M$.
\end{itemize}
\caption*{Algorithm $1$: {\bf SPD-OD}, outlier detection by a soft power diagram}\label{algo:outlier}
\end{table}

Without expert knowledge, a natural choice for the sites are the arithmetic means $c_i$ of the clusters $C_i$. A classical example for using the arithmetic means of clusters as their representative points is the well-known $k$-means algorithms. 

 In Table \ref{table:linearprograms}, we report on computation times for different values of $t$. As the data sets differ in size, we chose $t$ according to a common percentage of the theoretically maximal number of margin error points. The main part of Algorithm $1$ consists of the solution of a linear program, which explains the favorable running times. To make the short computation times comparable, we loaded the problem into memory once and then solved it ten times from scratch. Obviously the larger linear programs for the data sets (especially {\em shuttle}) took longer to solve, but our results do not reveal an immediate connection of the number of margin errors and the running time.

\begin{table}
\centering
    \renewcommand{\arraystretch}{1.5}
\begin{tabular*}{\textwidth}{@{\extracolsep{\fill}}| c | c | c | c | }
  \hline
  data set & $5\% $ margin errors & $15\% $ margin errors  & $30\% $ margin errors\\
  \hline
 dna  &  $\approx 3$ sec. &  $\approx 3$ sec. &   $\approx 3$ sec. \\
\hline
 vowel  &  $\approx 2$ sec. &  $\approx 4$ sec.  &  $\approx 2$ sec.\\
  \hline
 satimage &  $\approx 20$ sec. &  $\approx 22$ sec.  &   $\approx 23$ sec.\\
\hline
 shuttle  & $\approx 49$ sec. &  $\approx 48$ sec.  &  $\approx 50$ sec.\\
\hline
\hline
\end{tabular*}
\caption{Computation times for {\em ten} applications of Algorithm $1$. The programs were loaded into memory once and then solved ten times. }
\label{table:linearprograms}
\end{table}

\subsection{Least-squares-thresholds}

We now take a closer look at the tradeoff of the number of margin errors and the size of the margin in our programs. Let $(P_{\text{MME}}')(t)$ and $(P_{\text{MEP}}')(t)$ refer to using parameter $t$ in the programs, to obtain optimal objective function values $\Theta^*_{\text{MME}}(t)$ and $\Theta^*_{\text{MEP}}(t)$. Further, let $\epsilon^*_{\text{MME}}(t)$ and $\epsilon^*_{\text{MEP}}(t)$ be the corresponding optimal values for $\epsilon$. Finally, let both $\epsilon^*_{\text{MME}}(0)$ and $\epsilon^*_{\text{MEP}}(0)$ refer to the margin $\epsilon$ as computed in $(P_{\text{SPD}}')$.

An intuitive measure for the 'quality of separation' or separability of a clustering with respect to a soft power diagram is  the smallest bound $t$ on the number of margin errors which yield a nonnegative $\epsilon^*_{\text{MME}}(t)$ or $\epsilon^*_{\text{MEP}}(t)$: It is the smallest bound $t$ such that dropping the (at most $t$) margin errors yields a separating power diagram. Let us prove that the computation of this value is efficiently possible.

\begin{theorem}\label{thm:last}
Let $\mathcal{C}$ be a clustering of $X$, and let $S$ be given. Then it is possible to compute the minimal $t$ for which $\epsilon^*_{\text{MME}}(t) \geq  0$ by solving at most $\lceil \log ((k-1)n)\rceil$ programs of type $(P_{\text{MME}}')$ (and one of type $(P_{\text{SPD}}')$). Also, it is possible to compute the minimal $t$ for which $\epsilon^*_{\text{MEP}}(t)\geq  0$  by solving at most $\lceil \log n \rceil$ programs of type  $(P_{\text{MEP}}')$  (and one of type $(P_{\text{SPD}}')$). 
\end{theorem}
\begin{proof}
First, we check whether $(P_{\text{SPD}}')$ returns an objective function value of  $\epsilon \geq 0$. If so, there is a separating $S$-power diagram. We set $t=0$, and are done. In the claim, this refers to the parts in brackets.

For $t\geq 1$, we prove the claim by showing that $(\epsilon^*_{\text{MME}}(t))_{1\leq t \leq  (k-1) n}$ and $(\epsilon^*_{\text{MEP}}(t))_{1\leq t \leq n}$ are non-decreasing sequences which contain a non-negative value. Then we can use nested intervals to obtain the minimal $t$, and obtain the claimed number of programs. We start by turning to $(\epsilon^*_{\text{MEP}}(t))_{1\leq t \leq n}$. Only the range $1\leq t \leq n$ is of interest; at most all $n$ points are support vector points.

For a simple notation, let $f_t:=\frac{t+\frac{1}{2}}{t(t+1)}$. Let further $(\gamma^t,\epsilon^t, \xi^t):=(\gamma^*(t), \epsilon^*_{\text{MEP}}(t), \xi^*(t))$ be optimal for $(P_{\text{MEP}}')(t)$, and $\sum\xi^t:= \sum\limits_{l=1}^n \xi_{l}^*(t)$. Also, note that the feasibility region of  $(P_{\text{MEP}}')(t)$ is independent of $t$: We will use this fact to relate the objective function values of $(\gamma^t,\epsilon^t, \xi^t)$ and $(\gamma^{t+1},\epsilon^{t+1}, \xi^{t+1})$ with respect to $(P_{\text{MEP}}')(t)$ and $(P_{\text{MEP}}')(t+1)$:

Clearly, $(\Theta^*_{\text{MEP}}(t))_{1\leq t \leq n}$ is an increasing sequence: By $f_t>f_{t+1}$, we obtain $\Theta^*_{\text{MEP}}(t) < \Theta^*_{\text{MEP}}(t+1)$. Now suppose there is a $t\leq n-1$ such that $\epsilon^{t+1}<\epsilon^t$. Then $\sum\xi^t - \sum\xi^{t+1} \geq 0$, and then
  $$\epsilon^t-f_t\sum\xi^t \geq \epsilon^{t+1}-f_t\sum\xi^{t+1} \quad \Leftrightarrow \quad \epsilon^t-\epsilon^{t+1} \geq f_t(\sum\xi^t - \sum\xi^{t+1})\geq 0,$$ 
as well as $$\epsilon^{t+1}-f_{t+1}\sum\xi^{t+1} \geq \epsilon^t-f_{t+1}\sum\xi^t  \quad \Leftrightarrow \quad  \epsilon^t-\epsilon^{t+1} \leq f_{t+1}(\sum\xi^t - \sum\xi^{t+1}).$$
Combining this to 
$$f_{t+1}(\sum\xi^t - \sum\xi^{t+1}) \geq \epsilon^t-\epsilon^{t+1} \geq f_t(\sum\xi^t - \sum\xi^{t+1})\geq 0$$
is a contradiction with $f_t>f_{t+1}$ .

Thus $(\epsilon^*_{\text{MEP}}(t))_{1\leq t \leq n}$ is a non-decreasing sequence. By $f(t)<\frac{1}{t}$, we have $\Theta^*_{\text{MEP}}(n)=\infty$, i.e. the program is positively unbounded if we allow for the theoretically maximal number of margin error points. Thus there exists a minimal value of $t$ for which $\epsilon^*_{\text{MEP}}(t) \geq \Theta^*_{\text{MEP}}(t) \geq 0$. 

This proves termination of an algorithm which starts with $t=1$, and finds the minimal $t$ by nested intervals bisecting the current search ranges in $1,\dots, n$. In each iteration step, the corresponding linear program is solved. If $\epsilon^* \geq 0$, we terminate the algorithm. This happens after at most the claimed number of steps. 

The claim follows analogously for $(\epsilon^*_{\text{MME}}(t))_{1\leq t \leq  (k-1) n}$. Here we have to consider the range $1\leq t \leq  (k-1) n$, as there are in between one and $(k-1)n$ multiclass support vectors in a power diagram. At most all $n$ points can be multiclass support vectors with respect to all $k-1$ clusters they do not lie in. Except for using the sum $\sum\limits_{l=1}^n \sum\limits_{j\neq c(l)}  \xi_{jl}$, the above arguments remain the same. \qed
\end{proof}
Algorithm $2$ sums up this approach, for margin error points. It is denoted using a tweak for a more efficient implementation: Even though we have to compute solutions for up to $\lceil \log n \rceil$ linear programs, all of them  have the same feasibility regions. The only differences are in the  objective functions, and these are very similar in the later programs of a run of the algorithm. Thus, we keep the basis of active constraints of the preceding optimal solution as starting basis for the succeeding program, which then is solved in almost negligible time in later stages of the algorithm.

The value $\tau_{\text{MEP}}$, derived as the fraction of the final value of $t$ and $n$ -- which is the theoretically maximal number of margin error points --  is a nice measure for the separability of the underlying clustering with respect to a power diagram for the given sites.  The larger it is, the 'more different' the clustering $\mathcal{C}$ is from a (balanced) least-squares assignment with respect to the given sites. This sheds some insight into how well-chosen the sites are as representative points for their clusters. As a byproduct, we also obtain the power diagram itself.
\begin{table}
\begin{itemize}
\item {\bf Input}: $d,k,n \in \mathbb{N}$, clustering $\mathcal{C}$ of $X\subset\mathbb{R}^d$, sites $S\subset \mathbb{R}^d$
\item {\bf Output:} $\tau_{\text{MEP}}$, and the corresponding soft $S$-power diagram $P$
\item {\bf 1.} Solve the LP
\begin{eqnarray*}
\begin{array}{lrclcl}
\multicolumn{4}{c}{\max\, \epsilon}        &       &\\
  &\sigma_{ij}  + \epsilon & \leq &\gamma_{ij} &  &  (i \leq k, j \leq k: i\neq j).
\end{array}
 \end{eqnarray*}
If $\epsilon \geq 0$, return $\tau_{\text{MEP}}=0$ and the optimal solution $(\gamma,\epsilon,\xi)$ (a separating power diagram $P$). Otherwise set $t:=1$, $r:=1$ and set $(\gamma',\epsilon',\xi')=(0, -\infty, 0)$.
\item {\bf 2.} Find the optimal solution $(\gamma,\epsilon,\xi)$ (defining soft power diagram $P$)  for the LP
\begin{eqnarray*}
\max \Theta_{\text{MEP}}(t):= \epsilon -  \frac{t+\frac{1}{2}}{t(t+1)} \sum\limits_{l=1}^n \xi_{l}&&\\
\text{ s.t. }\quad s_{ij}^Tx_l  +\epsilon  \leq    \gamma_{ij} +\xi_{l} &&\quad  (l \leq n, j\leq k: j\neq i:=c(l))\\
\xi_{l}  \geq  0 &&\quad (l \leq n)
 \end{eqnarray*}
by using $(\gamma',\epsilon',\xi')$ as a starting solution.
\item {\bf 3.} If $r< \lceil \log n \rceil$, set $r:=r+1$. Depending on $\epsilon<0$ or $\epsilon \geq 0$, update $t$ to bisect the remaining interval of possible minimal values for $t$. Then set $(\gamma',\epsilon',\xi'):=(\gamma,\epsilon,\xi)$, and go to $2.$. Otherwise return $P$ and $\tau_{\text{MEP}}:=\frac{t}{n}$. 
\end{itemize}
\caption*{Algorithm $2$: {\bf LS-SPD}, threshold-setting soft power diagram}
\label{algo:SPD}
\end{table}
For a given clustering $\mathcal{C}$ and sites $S$, we call the fraction
$$\tau_{\text{MEP}}:=\frac{1}{n} \cdot ( \min\limits_{0 \leq t \leq n} t \text{ s.t. }\epsilon_{\text{MEP}}(t)\geq 0)$$
the {\bf MEP-threshold of $\mathcal{C}$ with respect to $S$}, and the fraction $$\tau_{\text{MME}}:=\frac{1}{(k-1)\cdot n}\cdot (\min\limits_{0 \leq t \leq (k-1)n} t \text{ s.t. }\epsilon_{\text{MME}}(t)\geq 0)$$
the {\bf MME-threshold of $\mathcal{C}$ with respect to $S$}. Analogously to the interpretation of $\tau_{\text{MEP}}$, $\tau_{\text{MME}}$ represents the  percentage of multiclass margin errors against the theoretically maximal number of multiclass margin errors.

Besides measuring the similarity to a balanced $S$-least-squares assignment, the thresholds give apriori information about the quality of a classifier. This is is in contrast to the aposteriori information obtained from applying a classifier to a test set. This is especially useful when a sufficiently large test set is not available. Let us turn to some computational results for this interpretation. We used the arithmetic means of the clusters as sites.

Table \ref{table:nusvm} lists the MEP-threshold, the percentage of misclassified points in the test set and the time required for the computation of the threshold.

\begin{table}
\centering
    \renewcommand{\arraystretch}{1.5}
\begin{tabular*}{\textwidth}{@{\extracolsep{\fill}}| c | c | c | c | }
  \hline
  data set & $\tau_{\text{MEP}}$ & misclassifications  & time elapsed\\
  \hline
 dna  & $\frac{111 }{1400} \approx 7.93\%$ & $\frac{130}{1186}\approx 10.96\%$ &  $< 1$ sec.\\
\hline
 vowel  & $\frac{399}{528} \approx 75.59\%$ & $\frac{316}{462}\approx 68.40\%$  &  $< 1$ sec.\\
  \hline
 satimage  & $\frac{615}{3194} \approx 19.25\%$ & $\frac{393}{2000}\approx 19.65\%$  &  $\approx 6$ sec.\\
\hline
 shuttle  & $\frac{3069}{30450} \approx 10.08\%$ & $\frac{1457}{14500}\approx 10.05\%$  &  $\approx 22$ sec.\\
\hline
\hline
\end{tabular*}
\caption{Empirical results for Algorithm $2$. The MEP-threshold and the percentage of misclassifications are highly correlated.}
\label{table:nusvm}
\end{table}

As expected, the four data sets indicate a high correlation in between the MEP-threshold and the number of misclassifications: The sets {\em dna},  {\em satimage} and {\em shuttle} put up good threshold values, and then consequently also a very low rate of misclassifications. In fact, these percentile numbers are even almost identical for the data sets {\em satimage} and {\em shuttle}. Conversely, the set {\em vowel} did not prove well-separable, with a very high threshold of almost $76\%$. The derived classifier then performed similarly bad. These (and all other) empirical results confirm the interpretation of the MEP-threshold as a viable apriori measure for the quality of the computed soft power diagram as a classifier. 

\subsection{Proof of concept for the general model}

While not at the core of our discussion, let us conclude our applications by turning to some computations of locally optimal solutions for $(P_{\text{SPD}})$ and $(P_{\text{MEP}})$. Using standard modeling techniques, we denoted both as unrestricted problems with penalty terms for violated constraints and then computed a local optimum starting from the arithmetic means of the clusters as beginning points. For $(P_{\text{MEP}})$, we chose $t$ according to a percentage of $10\%$ of the theoretically maximal number of margin errors and stopped iterating with a precision of $0.001$. Table \ref{table:final} reports on our computation times and the number of misclassifications for the derived classifiers.

As expected, we observed that the power diagram derived by $(P_{\text{SPD}})$ (with margin $\epsilon<0$) is significantly outperformed by the power diagram derived by $(P_{\text{MEP}})$ when used as a classifier for all data sets with low MEP-threshold (see Table \ref{table:nusvm}), i.e. those similar to a balanced least-squares assignment. After all, $(P_{\text{SPD}})$  is designed with separable clusters in mind and can easily be influenced by outliers, in contrast to $(P_{\text{MEP}})$. We also observed that $(P_{\text{MEP}})$ often terminated with a local optimum whose sites were close to the starting sites for the data sets with low MEP-threshold. Then the number of misclassifications was also similar.

\begin{table}
\centering
    \renewcommand{\arraystretch}{1.5}
\begin{tabular*}{\textwidth}{@{\extracolsep{\fill}}| c | c | c | c | c | }
 \hline
 & $(P_{\text{SPD}})$ & & $(P_{\text{MEP}})$ & \\
  \hline
  data set & misclassifications  & time elapsed & misclassifications & time elapsed\\
  \hline
 dna  &  $\frac{170}{1186}\approx 14.33\%$ & $\approx 7$ sec. & $\frac{128}{1186}\approx 10.79\%$ & $\approx 59$ sec.\\
\hline
 vowel  & $\frac{308}{462}\approx 66.67\%$  & $\approx 24$ sec. & $\frac{310}{462}\approx 67.10\%$ &  $\approx 198$ sec.\\
  \hline
 satimage  & $\frac{522}{2000}\approx 26.10\%$  & $\approx 33$ sec. &  $\frac{399}{2000}\approx 19.50\%$ &  $\approx 230$ sec.\\
\hline
 shuttle  & $\frac{2590}{14500}\approx 17.86\%$  & $\approx 58$ sec.  &  $\frac{1449}{14500}\approx 10.00\%$ &  $\approx 550$ sec.\\
\hline
\hline
\end{tabular*}
\caption{Empirical results and computation times for locally optimal solutions of $(P_{\text{SPD}})$ and $(P_{\text{MEP}})$ starting from the arithmetic means as sites.}
\label{table:final}
\end{table}

In principle, $(P_{\text{MEP}})$ (and $(P_{\text{MME}})$) may serve as a multiclass support vector machine, but they are not designed with efficiency in mind. The use of shared margins is 'more' than the state-of-the-art multiclass support vector machines do; these benefit from accepting this lack of information by obtaining convex programs. We plan to investigate regularizations (e.g. the use of the $1$-norm in place of the Euclidean norm) and a careful choice of the starting sites for our general model to test its performance in practice in our research.

\bibliographystyle{plain}
\bibliography{sample}
\end{document}